\documentclass{article}

    \usepackage[nonatbib,preprint]{neurips_2024}

\usepackage[utf8]{inputenc} 
\usepackage[T1]{fontenc}    
\usepackage[hyperindex,breaklinks]{hyperref}
\usepackage{url}            
\usepackage{booktabs}       
\usepackage{amsfonts}       
\usepackage{nicefrac}       
\usepackage{microtype}      
\usepackage[dvipsnames]{xcolor}         
\usepackage{amsmath,amsfonts,mathtools,amsthm,amssymb}

\usepackage[capitalise]{cleveref}
\usepackage{algpseudocode, algorithm} 
\usepackage{subcaption}
\usepackage{wrapfig}

\theoremstyle{plain}
\newtheorem{theorem}{Theorem}[section]

\newtheorem{lemma}{Lemma}[section]

\newtheorem{remark}{Remark}[section]

\newcommand{\WI}{\mathcal{W}_1}
\newcommand{\WII}{\mathcal{W}_2}
\newcommand{\R}{\mathbb{R}}
\newcommand{\Ex}{\mathbb{E}}

\title{\LARGE Combining Wasserstein-1 and Wasserstein-2 proximals: robust manifold learning via well-posed generative flows }

\author{%
Hyemin Gu \quad Markos A. Katouslakis \quad Luc Rey-Bellet \quad Benjamin J. Zhang\\
Department of Mathematics and Statistics \\
University of Massachusetts Amherst \\
\texttt{\{hgu,markos,luc,bjzhang\}@umass.edu}
}

\begin{document}

\maketitle

\begin{abstract}
We formulate well-posed continuous-time generative flows for learning distributions that are supported on low-dimensional manifolds through Wasserstein proximal regularizations of $f$-divergences. Wasserstein-1 proximal operators regularize $f$-divergences so that singular distributions can be compared. Meanwhile, Wasserstein-2 proximal operators regularize the paths of the generative flows by adding an optimal transport cost, i.e., a kinetic energy penalization. Via mean-field game theory, we show that the \emph{combination} of the two proximals is critical for formulating well-posed generative flows. Generative flows can be analyzed through optimality conditions of a mean-field game (MFG), a system of a backward Hamilton-Jacobi (HJ) and a forward continuity partial differential equations (PDEs) whose solution characterizes the optimal generative flow. For learning distributions that are supported on low-dimensional manifolds, the MFG theory shows that the Wasserstein-1 proximal, which addresses the HJ terminal condition, and the Wasserstein-2 proximal, which addresses the HJ dynamics, are both necessary for the corresponding backward-forward PDE system to be well-defined and have a unique solution with provably \emph{linear} flow trajectories. This implies that the corresponding generative flow is also unique and can therefore be learned in a robust manner even for learning high-dimensional distributions supported on low-dimensional manifolds. The generative flows are learned through \emph{adversarial training} of continuous-time flows, which bypasses the need for reverse simulation. We demonstrate the efficacy of our approach for generating high-dimensional images without the need to resort to autoencoders or specialized architectures.

\end{abstract}

\section{Introduction}

Continuous normalizing flows (CNF) are a class of generative models \cite{chen2018neural,grathwohl2018ffjord} based on learning deterministic ODE dynamics that transport a simple reference distribution to a target distribution. However, two challenges arise when applying CNFs for learning high-dimensional distributions: (a) the flows are heavily discretization or implementation dependent due to lack of a unique optimizer, and (b) they are not suitable for learning distributions supported on manifolds, as is typically the case in many high-dimensional datasets, as they need to be invertible in practice.  For example, \cite{finlay2020train,onken2021ot} have noted that without additional regularizers, normalizing flows can produce multiple valid flows, many of which are difficult to integrate. \cite{ott2020resnet} has noted that the test accuracy of neural ODEs are discretization dependent. Moreover, as normalizing flows need to be invertible, latent representations or specialized architectures are frequently needed to successfully generate from distributions that are supported on manifolds for even moderately high-dimensional datasets \cite{chen2018neural,grathwohl2018ffjord,finlay2020train,onken2021ot,yang2019potential}. The challenges of CNFs are fundamentally rooted in the mathematical fact that their associated optimization problems are \emph{ill-defined}, especially for distributions supported on low-dimensional manifolds.

In this paper, through rigorous theory and numerical experiments, we demonstrate that the {\textbf{composition of the Wasserstein-1 and Wasserstein-2 proximal regularizations}} is essential for formulating \emph{well-posed} optimization problems in the training of deterministic flow-based generative models.
Recent work \cite{finlay2020train,yang2019potential,onken2021ot} has shown that training generative flows via a Kullback-Leibler (KL) divergence regularized with a \textbf{Wasserstein-2 proximal} ($\mathcal{W}_2$) through the Benamou-Brenier formulation of optimal transport yields sample paths that are more regular. In fact, we prove via the theory of Hamilton-Jacobi equations and mean-field games that regularizing the KL divergence with a $\WII$ proximal produces optimal trajectories with \emph{constant velocity}, i.e., straight lines. The proof is based on connections of Hamilton-Jacobi equations to the compressible Euler equations in fluid mechanics, which lead to further interesting interpretations of generative flows. Such flows can be discretized with fewer time steps, saving on computational overhead and may learn the target distribution faster.





Learning distributions supported on manifolds, however, is still a challenge with the Wasserstein-2 proximal alone as the KL (or, more generally, $f$-) divergence is unable to compare distributions that are not absolutely continuous with respect to each other as they are   finite only for distributions that share the same support. We propose regularizing the KL (or $f$-) divergence with a \textbf{Wasserstein-1} ($\mathcal{W}_1$) proximal as well \cite{birrell2020f,gu2023GPA}. Applying a $\mathcal{W}_1$ proximal regularization of the KL allows us to compare \emph{mutually singular distributions}. The variational derivative of a $\WI$ regularized KL divergence is well-defined for \emph{any} perturbation in the space of probability distributions, including mutually singular ones, meaning that the resulting training procedure (i.e., gradient descent) is \emph{robust} for learning distributions supported on manifolds.

This composition of $\WI$ and $\WII$ proximals for improving normalizing flows is grounded in the theory of mean-field games. Generative flows can be roughly viewed as optimal control problems where a velocity field, i.e., a control, is learned so that a collection of particles will evolve from a reference distribution to a target distribution by optimizing an objective functional involving interactions of generated and real samples through various mean-field mechanisms that depend on the class of models, e.g. CNFs, score-based models or gradient flows \cite{zhang2023mean}. The \emph{optimality conditions} of these MFGs are a pair of coupled backward-forward partial differential equations (PDEs): a backward Hamilton-Jacobi equation whose solution describes the optimal velocity field, and a forward continuity equation describing the generation process. The backward-forward structure mirrors that of forward and inverse flows in generative modeling \cite{zhang2023mean}. The terminal condition is associated with the target distribution and choice of probability divergence.
We propose a training objective with the composition of $\WI$ and $\WII$ proximals whose well-posedness is guaranteed by the well-posedness of the MFG optimality conditions. The $\WII$ proximal yields a well-defined Hamilton-Jacobi equation, while the $\WI$ proximal of $f$-divergences provides a well-defined terminal condition for distributions supported on manifolds. We show that the MFG is well-posed via a uniqueness theorem where we show that if the solution of the MFG PDE system is smooth, then the solution is unique. This uniqueness and well-posedness property of the generative flow produces more \emph{robust} learning algorithms that are less dependent on implementation details.

\begin{wrapfigure}[16]{r}{0.4\textwidth} 
  \centering
    \includegraphics[width=0.39\textwidth]{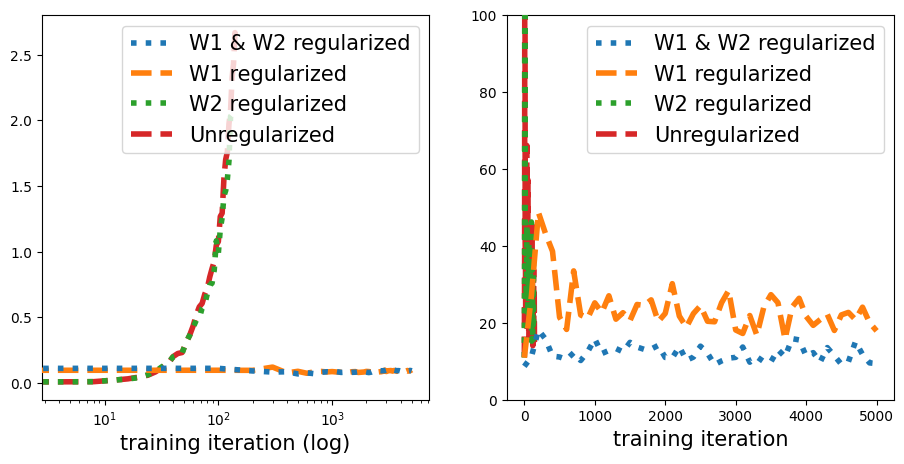}
    
  \caption{Optimality indicator \eqref{eq:indicator2} blows up without $\WI$ proximal (left). The composition of $\mathcal{W}_1$ and $\mathcal{W}_2$ proximals minimizes kinetic energy and keeps the value less oscillatory while training (right). See \cref{sec:numerical:examples} for details.}
  \label{fig:wrapfig}
\end{wrapfigure}

Our algorithm is based on adversarial training, where the proximal structure provides clear interpretations for the discriminators and generators. In contrast to previous CNF approaches \cite{finlay2020train,onken2021ot,grathwohl2018ffjord}, the adversarial formulation \emph{avoids} needing to invert the generative flow during training. Our resulting $\boldsymbol{\WI\oplus\WII}$ \textbf{proximal generative flows} robustly learns distributions supported on manifolds \emph{without specialized architectures or autoencoders}. We preview our numerical results in \Cref{fig:wrapfig}, which show the impact of $\WI$ and $\WII$ proximals for training generative flows of distributions supported on manifolds. Notice that without a $\WI$ proximal, the terminal KL cost diverges over the course of the training, while without a $\WII$ proximal, the kinetic energy oscillates during the course of the training, implying a lack of convergence to a unique flow.

\subsection{Contributions}

\begin{itemize}

\item We rigorously demonstrate that using the composition of Wasserstein-1 and Wasserstein-2 proximals of $f$-divergences to train generative flows is indispensable for \emph{robustly} learning distributions supported on low-dimensional manifolds. The $\mathcal{W}_2$ proximal guarantees trajectories are linear and the $\mathcal{W}_1$ proximal robustly learns distributions on manifolds.  These properties are due to the well-posedness of the optimization problem, which is justified via the theory of mean-field games and the properties of $\mathcal{W}_1$ and $\mathcal{W}_2$ proximals.

\item We produce an \emph{adversarial training algorithm} for training normalizing flows based on the dual formulation of the Wasserstein-1 proximal of $f$-divergences. This eliminates the need for forward-backward simulations of trajectories that are needed to train other CNFs. The call our resulting generative model $\boldsymbol{\WI\oplus \WII}$ \textbf{proximal generative flows}. 

\item We find that $\WI\oplus \WII$ generative flows \emph{robustly} learn high-dimensional data supported on low-dimensional manifolds \emph{without} pretraining autoencoders or tuning specialized architectures.

\end{itemize}

\subsection{Related work}

The use of Wasserstein proximals in generative modeling has already frequently appeared in previous studies. Using $\mathcal{W}_2$ regularized KL divergences for training CNFs were first introduced in \cite{finlay2020train} and \cite{onken2021ot}. The quality of their image generation models, however, were still implementation dependent. In \cite{yang2019potential}, the flows are regularized with the Hamilton-Jacobi equation corresponding to a $\mathcal{W}_2$ proximal, although their objective function does not include the kinetic energy penalization. In all these previous works, the HJ terminal conditions are ill-defined when the target distributions are supported on low-dimensional manifolds, and therefore depend on using some sort of autoencoder. $\mathcal{W}_2$ proximals of linear and nonlinear functions were generally studied in \cite{li2023kernel}. Score-based generative models were shown to be Wasserstein-2 proximals of cross-entropy in \cite{zhang2024wasserstein} using ideas from \cite{li2023kernel} and \cite{zhang2023mean}. Using adversarial training for normalizing flows was considered in \cite{grover2018flowGAN}, but without a kinetic energy penalization (i.e., a Wasserstein-2 proximal) a Jacobian normalization is employed to regularize training. The key connection between mean-field games and generative flows \cite{zhang2023mean} is a crucial tool that is used to prove our proposed $\WI\oplus\WII$ are trained via well-posed optimization problems.

The construction of Wasserstein gradient flow based the $\WI$ proximal of $f$- divergences for generative modeling were considered in \cite{gu2023GPA}. There, a stable particle transport algorithm was constructed that is able to flow between two mutually singular empirical distributions. Additional results about Wasserstein-1 regularized $f$- and R\'enyi divergences can be found in \cite{birrell2020f,birrell2023functionspace}. Their use for designing generative adversarial nets was studied in \cite{birrell2022structure}.

\section{Wasserstein proximals}

\label{sec:wasserstein:proximals}
Let $\mathcal{P}_p(\R^d)$ be the Wasserstein-$p$ space, the space of probability distributions on $\R^d$ with finite $p$-th moments with the Wasserstein-$p$ distance. Wasserstein-1 and Wasserstein-2 proximals have been been used intentionally or inadvertently for studying and designing generative models \cite{birrell2020f,birrell2023functionspace,finlay2020train,gu2023GPA,li2023kernel,onken2021ot,yang2019potential,zhang2024wasserstein}. Wasserstein-$p$ proximals for $p=1,2$ regularize cost functions (e.g., divergences) $\mathcal{F}: \mathcal{P}_p(\mathbb{R}^d) \rightarrow \mathbb{R}$ by an infimal convolution between the function $\mathcal{F}(R)$ and the $p$-th power of Wasserstein-$p$ distance $\mathcal{W}_p^p (P ,R)$ from the input measure $P$  
\begin{equation}
\label{eq:general:proximal:regularization}
   \inf_{R \in \mathcal{P}_p(\mathbb{R}^d)} \left\{\mathcal{F}(R) + \theta \cdot \mathcal{W}_p^p (P ,R)  \right\}
\end{equation}
where $\theta >0$ is a weighting parameter. We define the Wasserstein proximal operator $\textbf{prox}_{\theta,\mathcal{F}}^{W_p}:\mathcal{P}_p(\R^d) \to \mathcal{P}_p(\R^d)$ to be the mapping from input measure $P$ to the optimizer of the \eqref{eq:general:proximal:regularization}, $R^*$. That is, 
\begin{equation}
    \label{eq:proximal:operator}
    R^* = \textbf{prox}_{\theta ,\mathcal{F}}^{\mathcal{W}_p}(P) := \arg\inf_{R \in \mathcal{P}_p(\mathbb{R}^d)} \left\{\mathcal{F}(R) + \theta \cdot \mathcal{W}_p^p (P ,R)  \right\}.
\end{equation}
The Wasserstein proximal operators for $\mathcal{P}_p(\R^d)$, are the analogues of proximal operators that arise in convex optimization on $\R^d$
\cite{ParikhBoyd} \footnote{Let $f:\R^d \to \R$ be a semicontinuous convex function. The proximal operator of $f$ is a generalization of a single gradient flow step for $f$ using a backward Euler scheme with stepsize $1/\lambda$. For $\R^d$, $\textbf{prox}_{\lambda ,f}(v) = \arg\min_{x \in \R^d} \left[ f(x) + \frac{\lambda}{2} \|x-v \|_2^2 \right]$. Proximal operators are useful for optimizing discontinuous or non-differentiable functions $f$.   }. In finite dimensional optimization, proximal operators are useful for developing optimization methods for nonsmooth or discontinuous functions. Analogously, the \emph{Wasserstein} proximal operators are used to optimize nonsmooth functionals over Wasserstein-$p$ space.

We first review how the Wasserstein-2 and Wasserstein-1 proximals are defined and implemented in existing generative models, and how the composition of the proximals is defined.

\subsection{Wasserstein-2 proximal stabilizes training of generative flows }
\label{subsec:background:w2}
In canonical continuous normalizing flows, we find a velocity field $v$ such that the differential equation $\frac{dx}{dt} = v(x(t),t)$ flows a reference measure $\rho(\cdot,0) = \rho_0$ to a target measure $\pi(\cdot)$ given only samples of the latter. The velocity field is learned by optimizing a divergence, such as the KL divergence so that the terminal velocity field $\rho_T = \rho(\cdot,T)$ well-approximates $\pi$. In this case, the cost function is defined $\mathcal{F}(\rho_T) = \mathcal{D}_{KL}(\rho_T\|\pi) = \Ex_{\rho_T} \left[\log \frac{d\rho_T}{d\pi} \right]$. The Wasserstein-2 proximal of $\mathcal{F}$ is then 
\begin{align}
    \inf_{\rho_T} \left\{ \mathcal{F}(\rho_T) + \frac{\lambda}{2T} \mathcal{W}_2^2(\rho_0,\rho_T)\right\}. \label{eq:w2:proximal}
\end{align}
In relation to \eqref{sec:wasserstein:proximals}, here $\theta = \lambda/(2T)$. By the Benamou-Brenier dynamic formulation of optimal transport, the continuous time optimization problem of \eqref{eq:w2:proximal} is 
\begin{align} \label{eq:w2:brenier:bernamou}
    \inf_{\rho_T,v}\left\{ \mathcal{F}(\rho_T) + \lambda\int_0^T \int_{\R^d} \frac{1}{2}|v(x,t)|^2 \rho(x,t) dx dt: \partial_t \rho + \nabla \cdot(v\rho) = 0,\, \rho(x,0) = \rho_0(x)\right\}. 
\end{align}
This is precisely the optimization problem considered in \cite{onken2021ot} and \cite{finlay2020train} for regularizing the training of continuous normalizing flows. The optimal transport cost can be interpreted as a kinetic energy which favors straighter paths and makes the optimization more well-posed as it eliminates many valid flows that may be difficult to integrate. In fact, in Theorem \ref{thm:main} we show that the optimal trajectories are \emph{exactly linear}. 

Previous computational studies \cite{finlay2020train,onken2021ot} showed that optimal transport regularizations stabilized training of CNFs and significantly reduced the training time. Learning distributions supported on low-dimensional manifolds, however, is not straightforward, and previous work implemented specialized architectures or preprocessed the data with autoencoders.

\subsection{Wasserstein-1 proximal enables manifold learning}
\label{subsec:background:w1}
The $f$-divergence between two distributions $\rho$ and $\pi$ is defined $D_f(\rho\|\pi) = \mathbb{E}_{\pi}\left[f\left( \frac{d\rho}{d\pi}\right) \right]$, where $\frac{d\rho}{d \pi}$ is the Radon-Nikodym derivative. A key weakness of $f$-divergences is that they are only finite when the likelihood ratio $d\rho/d\pi$ is finite, meaning that they are unable to effectively learn discrepancies when $\rho$ and $\pi$ are not supported on the same manifold. To address this deficiency, \cite{birrell2020f} studied a broad class of regularized $f$-divergences called the $(f,\Gamma)$-divergence, denoted $D_f^\Gamma$, defined via the infimal convolution of the $f$-divergence with an integral probability metric (IPM) over function space $\Gamma$. Define $\mathcal{W}_\Gamma$ as a general integral probability metric (IPM) over function space $\Gamma$, $
    \mathcal{W}_\Gamma(\rho,\nu) = \sup_{\phi \in \Gamma} \left\{\Ex_\rho [\phi] - \Ex_\nu [\phi] \right\}.$
The main benefit of IPMs is that they are able to compare distributions with different supports. The $D_f^{\Gamma}$ divergence with weighting parameter $L>0$ is 
\begin{align}
    D_f^{\Gamma}(\rho\|\pi) = \inf_{\nu \in \mathcal{P}_1(\mathbb{R}^d)} \left\{ D_f(\nu\|\pi) + L\cdot\mathcal{W}_\Gamma(\rho,\nu )\right\}.
\end{align}
 Via the dual formulation, the Wasserstein-1 distance is an IPM with $\Gamma$ being the space of $1$-Lipschitz functions. Wasserstein-1 proximals of $D_f$ are a special $(f,\Gamma)$-divergence called \emph{Lipschitz-regularized} $f$-divergences, whose primal and dual formulations are
\begin{align}
    D_f^{\Gamma_L}(\rho \| \pi)  &= \inf_{\nu \in \mathcal{P}_1(\mathbb{R}^d)} \left\{D_f(\nu\|\pi) + L\cdot \WI(\rho,\nu) \right\} \label{eq:kl:lip:divergence:primal} \\
    & = \sup_{\phi \in \Gamma_L}\left\{\Ex_\rho[\phi] - \Ex_\pi[f^\star(\phi)] \right\},\label{eq:kl:lip:divergence:dual}
\end{align}
where $\Gamma_L$ is the space of $L$-Lipschitz functions and $f^\star$ is the convex conjugate of $f$. Lipschitz-regularized $f$-divergences are also related to the so-called Moreau-Yosida $f$-divergences \cite{terjek21amoreauyoshida}.

\paragraph{Manifold detecting properties.}

By definition of its primal formulation \eqref{eq:kl:lip:divergence:primal}, the new divergence  satisfies the following fundamental inequality
 \begin{equation}\label{eq:fgdivergence:bounds}
     0\leq D_{f}^{\Gamma_L}(\rho\|\pi) \leq\min\left\{ D_f(\rho\|\pi),L \cdot W_1(\rho,\pi)\right\}\, .
 \end{equation} 
 The inequality implies that, unlike $f$-divergences, the new divergence $D_{f}^{\Gamma_L}(\rho\|\pi)$ is able to compare measures $\rho, \pi$ that are mutually singular. This is especially important when we approximate a target distribution $\pi$ by the generative distribution $\rho$. When $\pi$ is supported on an unknown \textit{manifold}, which is typical in high-dimensional datasets such as images, the generative distribution $\rho$ can easily be singular with respect to $\pi$.

 \paragraph{Lipschitz-regularized $f$-divergences are smooth.}
Another key property of $D_{f}^{\Gamma_L}(\rho\|\pi)$ is \emph{the existence of a variational derivative} with respect to $\rho$, for any $\rho, \pi\in 
\mathcal{P}_1$, {\it without any absolute continuity or density assumptions} on either $\rho$ or $\pi$
\cite{gu2023GPA}. In addition, the optimizer $\phi^*$ of \eqref{eq:kl:lip:divergence:dual} is the variational derivative of the Lipschitz regularized $f$-divergence 
\begin{equation}
\label{eq:first:variation}
    \frac{\delta D_f^{\Gamma_L}(\rho\|\pi)}{\delta \rho} = \phi^*=
    \arg\max_{\phi \in \Gamma_L}  \left \{ \mathbb{E}_\rho[\phi] - \mathbb{E}_\pi [f^\star(\phi)] \right \}\, ,
\end{equation}
A priori the maximizer on the r.h.s of \eqref{eq:first:variation} is defined 
uniquely, up to an additive constant, in the support of $\rho$ and $\pi$ but the varational derivative should be understood as the $L$-Lipschitz regularization of the optimizer which is then defined and $L$-Lipschitz continuous everywhere in $\mathbb{R}^d$, see Lemma 2.3 and Remark 2.4  in \cite{gu2023GPA}.  
In the context of adversarial training, $\phi^*$ can be interpreted as the \emph{discriminator} \cite{gu2023GPA,birrell2022structure}. This fact was first proved in \cite{gu2023GPA} (Theorem 2.1), and was used to construct data-driven Wasserstein gradient flows of Lipschitz-regularized $f$-divergences. 

The existence of a unique well-defined variational derivative for \textit{any} perturbations of all measures in $\mathcal{P}_1$ implies \emph{robust} distribution learning even for distributions supported on a low-dimensional manifold. Gradient-based optimization methods implicitly approximate the variational derivative of the cost functional. IPMs, such as $\WI$, do not have variational derivatives, while $f$- divergences cannot handle distributions with differing supports. \emph{The Wasserstein-1 proximal regularization of $f$-divergences inherits the beneficial qualities of the $\WI$ distance and $f$-divergences.}

\section{Combining Wasserstein-1 and Wasserstein-2 proximals}
\label{subsec:composition:wasserstein:proximals}
The main contribution of this paper is the use of both $\WI$ and $\WII$ proximals of $f$-divergences to train normalizing flows that can learn distributions supported on manifolds. The $\WI$ proximal of $f$-divergences produces Lipschitz-regularized $f$-divergences $D_{f}^{\Gamma_L}$ that enables comparing distributions with differing supports. The $\WII$ proximal of $D_f^{\Gamma_L}$ stabilizes the training of the generative flow by making the optimization problem well-posed. Let $\rho_0$ be the reference distribution, $\pi$ be the target distribution and $\rho_T$ be the approximating distribution to $\pi$. The combined Wasserstein-1 and Wasserstein-2 proximal of $f$-divergences is defined as
\begin{align}
\label{eq:terminal:running:costs}
     &\inf_{\rho_T}  \left\{ D_f^{\Gamma_L} (\rho_T \| \pi) + \frac{\lambda}{2T} \cdot \WII^2 (\rho_0, \rho_T)   \right\} \\
    \label{eq:composed:proximal}
     =& \underbrace{\inf_{\rho_T}  \Bigg\{ \underbrace{ \inf_\sigma \left\{  D_f (\sigma \| \pi) + L \cdot \WI (\rho_T, \sigma) \right \}}_{D_f^{\Gamma_L} = \WI \text{ proximal of } D_f} + \frac{\lambda}{2T} \cdot \WII^2 (\rho_0, \rho_T)  \Bigg\}}_{\WII \text{ proximal of } D_f^{\Gamma_L} } \, .
\end{align}

As is, it is not clear how one may practically compute the objective functional \eqref{eq:composed:proximal}. We describe a computable form of this optimization problem if we consider optimizing for a continuous time flow. The $\WII$ distance is replaced with the Benamou-Brenier formulation \eqref{eq:w2:brenier:bernamou}. Furthermore, to compute $D_f^\Gamma$, we use its dual formulation 
\eqref{eq:kl:lip:divergence:dual} as it is implemented in \cite{birrell2020f,gu2023GPA}. The resulting optimization problem is
\begin{align}
    \inf_{v, \rho}  \Bigg\{ \overbrace{\sup_{\phi \in \Gamma_L}  \left \{ \mathbb{E}_{\rho(\cdot, T)}[\phi] - \mathbb{E}_\pi [f^\star(\phi)] \right \} }^{\text{Dual formulation 
    of $D_f^{\Gamma_L}$ }} + \overbrace{\lambda \int_0^T \int_{\mathbb{R}^d} \frac{1}{2} |v(x,t)|^2 \rho(x,t) dx dt}^{\text{Benamou-Brenier dynamic OT formulation}} \Bigg\} \label{eq:our:objective} \\
    \frac{dx}{dt} = v(x(t),t),\, x(0) \sim \rho_0,\, t\in[0,T].     \nonumber
\end{align}
Again, note that $\Gamma_L$ is the space of $L$-Lipschitz continuous functions and $f^\star$ is the convex conjugate of $f$. Roughly speaking, given a proposed velocity field and samples from $\pi$, the objective functional \eqref{eq:our:learning:objective} is computed by running an ensemble of trajectories with initial conditions starting from reference distribution $\rho_0$. The optimal transport cost is approximated using trajectories. The final location of the ensemble determines $\rho(\cdot,T)$, which is compared with samples from $\pi$ by optimization for a $\phi$ given by the dual formulation of $D_f^\Gamma$. In Section \ref{sec:algorithm:and:indicator}, we will outline the algorithm for computing the objective function \eqref{eq:our:learning:objective} in detail. We call the resulting flows trained on this optimization problem $\boldsymbol{\WI\oplus\WII}$ \textbf{generative flows}.

Through a generative adversarial network perspective, we call the optimal $\phi^*$, defined in \eqref{eq:first:variation}, the \emph{discriminator} that evaluates the similarity between the flow-\emph{generated} distribution $\rho(\cdot,T)$ and the target distribution $\pi$. The discriminator exists uniquely up to an additive constant, is well-defined and $L$-Lipschitz on all of $\R^d$. 

We describe further properties of the discriminator $\phi^*$ in Remark~\ref{rmk:composedprox}, and we refer the reader to \cite{birrell2020f,gu2023GPA} for further technical details. To connect with the notion of proximal operators in Wasserstein space \eqref{eq:proximal:operator} \cite{ParikhBoyd,li2023kernel}, we note that the reference distribution $\rho_0$ is related to the generated distribution $\rho_T$ by
\begin{align}
    \rho_T & =  \textbf{prox}_{\frac{\lambda}{2T}, D_f^{\Gamma_L} (\cdot\|\pi)}^{\WII}(\rho_0)= \arg\min_{\nu \in \mathcal{P}_2(\mathbb{R}^d)} \left\{ D_f^{\Gamma_L} (\nu\|\pi) + \frac{\lambda}{2T} \cdot \WII^2 (\rho_0 ,\nu) \right \}. \label{eq:proximaloperatorflow}
\end{align}

\begin{remark}[Composition of Wasserstein proximals and its interpretation] \label{rmk:composedprox}
One may notice that while we use two Wasserstein proximals to construct the $\WI\oplus\WII$ generative flow, there is only a single ($\WII$) proximal operator that relates $\rho_T$ to $\rho_0$ in \eqref{eq:proximaloperatorflow}. There is an implicit second $\WI$ proximal operator being computed when calculating the $D_f^\Gamma$ divergence. In addition to $\rho_0$, $\rho_T$, and $\pi$, there is a fourth distribution $\sigma^*$ that is produced when comparing $\rho_T$ with $\pi$. We have
\begin{align}
    \sigma^* = \textbf{prox}_{L, D_f (\cdot\|\pi)}^{\WI} (\rho_T) = \arg\min_{\sigma \in \mathcal{P}_1(\mathbb{R}^d)} \left\{ D_f (\sigma\|\pi) + L \cdot \WI (\rho_T ,\sigma) \right \},
    \end{align}
where $\sigma^*$ can be interpreted as an intermediate measure that is close to $\rho_T$ in the $\WI$ distance while sharing the same support as $\pi$ so that they can be compared via the $f$-divergence. Moreover, by \eqref{eq:proximaloperatorflow}, we see that $\sigma^*$ is equal to a \emph{composition} of the $\WI$ and $\WII$ proximal operators
    \begin{align}
        \sigma^* = \left(\textbf{prox}_{L, D_f (\cdot\|\pi)}^{\WI} \circ \textbf{prox}_{\frac{\lambda}{2T},  D_f^{\Gamma_L} (\cdot\|\pi)}^{\WII} \right) ( \rho_0 ). 
    \label{eq:sigma:double:proximals}
    \end{align}
The measure $\sigma^*$ is not explicitly computed in our algorithm as we opt to work with the dual formulation of the $\WI$ proximal of $D_f$. However, $\sigma^*$ is related to the optimal discriminator $\phi^*$ by the relation
\begin{align}
   \frac{d\sigma^*}{d \pi}(x) = (f^\star)'(\phi^*(x) - c )
\end{align}
for some constant $c$. That is, the likelihood ratio of $\sigma^*$ with respect to the target distribution $\pi$ is equal to the derivative of $f^\star$ evaluated at a constant shift of the optimal discriminator $\phi^*$. See \cite{birrell2020f}, Theorem 25 for further details.

\end{remark}

\section{$\WI\oplus \WII$ generative flows as mean-field games}
\label{sec:W1:W2:generative:flow}
In this section, we study our new $\WI\oplus\WII$ generative flow as a mean-field game (MFG) and discuss the conditions for the combined proximal training objective to be \emph{well-posed}. We will see that the well-posedness of our optimization can be understood through MFG theory. After reviewing the basics of MFG theory for generative flows, we prove that the $\WI\oplus\WII$ generative flows have a well-defined PDE system even for high-dimensional data supported on manifolds. The PDE solutions are \emph{unique} assuming that the solution is smooth, which implies the optimization problem is well-posed, and yields a \emph{unique optimal generative flow}. This uniqueness property of the generative flow is expected to render  corresponding algorithms more stable and not subject to  implementation-dependent outcomes.  Moreover, we show that the optimal trajectories are linear through connections to PDE theory and fluid mechanics.

\subsection{Background on continuous-time generative flows as mean-field games}
\label{subsec:background:mfg}
Generative flows were shown to be trained by solving mean-field game problems in \cite{zhang2023mean}. In a general sense, the mean-field game for deterministic generative flows is given by a terminal cost functional $\mathcal{F}: \mathcal{P}(\R^d) \to \R$, such as the KL divergence with respect to target distribution $\pi$, which measures the discrepancy between $\pi$ and approximating distribution $\rho(\cdot,T)$, and a convex running cost $L(x,v):\R^d\times \R^d \to \R$, which measures the cost particle moving at a particular velocity $v$ at position $x$. The mean-field game is
\begin{align} \label{eq:mfg}
\begin{split}
&\min_{v, \rho}  \, \mathcal{F}(\rho(\cdot, T))  + \int_{0}^T \int_{\mathbb{R}^d} L(x, v(x,t)) \rho(x,t) dx dt \\
\text{s.t. }& \partial_t\rho + \nabla \cdot (v \rho) = 0\,\,\,\, \rho(x, 0) = \rho_0(x)
\end{split}
\end{align}
Via a Lagrangian interpretation, the continuity equation can be replaced by particle trajectories that the differential equation $dx/dt = v(x(t),t)$, where $x(t) \sim \rho(\cdot,t)$. The advantage of the MFG perspective is that its associated \emph{optimality conditions}, which are in the form of a system of backward-forward PDEs, describe the mathematical structure of the optimal velocity field and density evolution \emph{a priori} to computation. Moreover, the lack of optimality conditions implies \emph{ill-posedness} of the optimization problem.

Define the Hamiltonian of the MFG to be the convex conjugate of the running cost function $L$, $H(x,p)= \sup_{v }\left[ -p^\top v - L(x,v) \right]$. The optimizers of \eqref{eq:mfg} are characterized by the optimality conditions
\begin{align}
    \begin{dcases}
    - \frac{\partial U}{\partial t} + H(x,\nabla U) = 0, \,\,\,  & U(x,T) =  \frac{\delta \mathcal{F}(\rho(\cdot,T))}{\delta \rho(\cdot, T)}(x), \\
      \frac{\partial \rho}{\partial t} - \nabla \cdot(\nabla_p H(x,\nabla U) \rho) = 0, \,\,\,\,\, & \rho(x,0) = \rho_0(x). 
    \end{dcases}
    \label{eq:optimality:conditions:general}
\end{align}
The first equation is a backward Hamilton-Jacobi (HJ) equation, while the second equation is the forward continuity equation with the optimal velocity field is determined by the solution to the HJ equation $v^* = -\nabla_p H(x,\nabla U)$. The well-posedness of the optimization problem for training CNFs \eqref{eq:mfg} can then be studied via PDE theory of the optimality conditions \eqref{eq:optimality:conditions:general}.

Empirically, the optimization problem for learning canonical CNFs, i.e., when $\mathcal{F}(\rho)$ is the KL divergence and $L(x,v) = 0$, is known to exhibit instabilities during training as there are multiple  flows that can minimize the objective functional \cite{ott2020resnet,finlay2020train,onken2021ot}. This fact can be immediately seen through MFGs since with $L = 0$, there is no well-defined Hamiltonian, meaning there is no HJ equation, and therefore no optimality conditions. It is also easy to see that there are infinitely many optimal flows in this setting. See \cite{zhang2023mean} for further details.  On the other hand, with an optimal transport cost $L(x,v) = \frac{1}{2}|v|^2$, \cite{onken2021ot,finlay2020train} found that training stabilized, and \cite{onken2021ot} derived the associated Hamilton-Jacobi equation and used it to further accelerate the training process.

\subsection{MFG for $\WI\oplus \WII$ generative flows have well-defined optimality conditions and imply \emph{linear} optimal trajectories}

Wasserstein-2 proximals stabilize the training of generative flows, while Wasserstein-1 proximals enables manifold learning. In this section, we back up these claims by studying the PDE system that arises from their optimality conditions. In particular, we will see the $\WII$ proximal yields a well-defined Hamilton-Jacobi equation for a potential function $U(x,t)$, while the $\WI$ proximal produces well-defined terminal condition $U(x,T)$.

\begin{theorem}[MFG for $\WI\oplus\WII$ proximal generative flow]
\label{thm:main}
Let $\pi$ be an unknown arbitrary target measure, $\rho(x, 0) = \rho_0$ be a given reference measure, and $v: \mathbb{R}^d \times \mathbb{R} \rightarrow \mathbb{R}^d$ be a vector field. Fix a terminal time $T >0$ and $\lambda >0$. The optimization problem 
\begin{equation}
\label{eq:our:learning:objective}
    \inf_{v, \rho} \left\{ \sup_{\phi \in \Gamma_L}  \left \{ \mathbb{E}_{\rho(\cdot, T)}[\phi] - \mathbb{E}_\pi [f^\star(\phi)] \right \} + \lambda \int_0^T \int_{\mathbb{R}^d} \frac{1}{2} |v(x,t)|^2 \rho(x,t) dx dt \right\}
\end{equation}
where $\rho(x,t)$ satisfies the continuity equation $\partial_t \rho + \nabla \cdot (\rho v) = 0,$ $\rho(x,0) = \rho_0(x)$ has the following optimality conditions: 
\begin{enumerate}
    \item \textbf{Wasserstein-2 proximal yields a well-posed Hamilton-Jacobi equation. } The optimal velocity $v^*$ is given by $v^* = -\frac{1}{\lambda}\nabla U$ where $U$ and $\rho$ satisfy a system consisting of a backward Hamilton-Jacobi equation and a forward continuity  equation for $ t \in [0,T]$
    \begin{align}
    \begin{dcases}
    - \partial_t U + \frac{1}{2\lambda}|\nabla U|^{2} = 0 \,\,\,\,\, & \\
      \partial_t \rho - \nabla \cdot \left(\rho \frac{\nabla U}{\lambda} \right) = 0 \,\,\,\,\, & \text{where} \,\,\,\, \rho(x,0) = \rho_0(x). 
    \end{dcases}
    \label{eq:optimality:conditions1}
\end{align}

\item \textbf{Wasserstein-1 proximal provides a well-defined terminal condition. } The dual formulation and the variational derivative of the Wasserstein-1 proximal  $D_f^{\Gamma_L}(\rho(\cdot, T) \| \pi)$, 
\eqref{eq:kl:lip:divergence:dual} and \eqref{eq:first:variation} respectively, provide the terminal condition of the Hamilton-Jacobi equation given as
\begin{equation}
\label{eq:optimality:conditions2}
    U(x, T) = \frac{\delta D_f^{\Gamma_L}(\rho(\cdot, T) \| \pi)}{\delta \rho(\cdot, T)}(x) =  \arg\max_{\phi \in \Gamma_L}  \left \{ \mathbb{E}_{\rho(\cdot, T)}[\phi] - \mathbb{E}_\pi [f^\star(\phi)] \right \} = \phi^*(x).
\end{equation}

\item \textbf{Optimal trajectories are linear.}
The optimal velocity field is $v(x,t) = -\frac{1}{\lambda} \nabla U(x,t)$. For an initial condition $x(0) \sim \rho(\cdot,0)$, the optimal trajectories have the following representations
\begin{align} \label{eq:trajectories}
    x(t)& = x(0) - \frac{1}{\lambda} \int_0^t \nabla U(x(s),s) ds =   x(0) - \frac{t}{\lambda} \nabla U(x(0),0) \nonumber \\
    & = x(T) - \frac{1}{\lambda}\int_T^t \nabla U(x(s),s) ds   = x(T) + \frac{T-t}{\lambda} \nabla U(x(T),T) \\ & = x(T) + \frac{T-t}{\lambda}\nabla\phi^*(x(T)) \nonumber
\end{align}
meaning that they are \emph{exactly} linear. 
\end{enumerate}

\end{theorem}
\begin{proof}
    The proof is a direct application of the general MFG optimality conditions in \eqref{eq:optimality:conditions:general} \cite{zhang2023mean}. Observe that the Hamiltonian is
    \begin{align}
        H(x,p) = \sup_{v \in \R^d} \left\{ -p^\top v - \frac{\lambda}{2} |v|^2 \right\} = \frac{1}{2\lambda}|p|^2. 
    \end{align}
    A key element of this proof is the use the variational derivative \eqref{eq:first:variation} of the Wasserstein-1 proximal $D_f^{\Gamma_L}(\rho(\cdot, T) \| \pi)$, which exists for any probability measures in $\mathcal{P}_1$ without any absolute continuity or density assumptions, and is equal to the discriminator $\phi^*$, as discussed in Section~\ref{subsec:background:w1} and \cite{gu2023GPA}. This terminal condition is $L$-Lipschitz continuous everywhere in $\R^d$. 

    To see that the optimal trajectories are linear, consider the (system of) PDEs associated with the optimal velocity field $v(x,t) = -\frac{1}{\lambda} \nabla U(x,t)$, which is derived from the Hamilton-Jacobi equation 
    \begin{align}
        -\frac{1}{\lambda}\partial_t \nabla U + \frac{1}{\lambda^2} \nabla U \cdot \nabla(\nabla U)= 0 
        \implies  \partial_t v + (v\cdot \nabla) v = 0. \label{eq:velocitypdes}
    \end{align}
    We implicitly assume that $U$ is sufficiently differentiable. See \cref{rmk:classicalsolution} for further comments on regularity. Note that together with the continuity equation in \eqref{eq:optimality:conditions1}, this system of PDEs arises in inviscid compressible fluid mechanics and is called the Euler equations \cite{gangbo2009optimal}. 
    Next, consider the \emph{characteristics} for the PDEs of $v$, i.e., the level curves on which $v$ is constant \cite{evans2022partial}. Let $y(t)$ be the characteristics which satisfy
    \begin{align}
        \frac{d}{dt} v(y(t),t) = \partial_t v(y(t),t) + \partial_t y(t) \cdot \nabla v(y(t),t) = 0.
    \end{align}
    However, from \eqref{eq:velocitypdes}, we know that $y(t)$ is a characteristic only if 
    \begin{align}\partial_t y(t) =v(y(t),t) = -\frac{1}{\lambda} \nabla U(y(t),t)\end{align} which is exactly the evolution of the optimal trajectories $x(t)$ in \eqref{eq:trajectories}. This means that the optimal trajectories are precisely the characteristics of the velocity field PDEs, implying that the \emph{velocity is constant along the optimal trajectories}. Therefore, $v(x(t),t) = v(x(0),0) = -\frac{1}{\lambda} \nabla U(x(0),0)$ for all $t\in[0,T]$, and so 
    \begin{align}
        x(t) = x(0) + \int_0^t v(x(s),s) ds = x(0) - \frac{t}{\lambda}\nabla U(x(0),0). 
    \end{align}
The optimal trajectories may instead be characterized by with their terminal endpoints. Along optimal trajectories, $v(x(t),t) = v(x(T),T) = -\frac{1}{\lambda} \nabla U(x(T),T)$. The terminal condition is equal to the discriminator, i.e., $U(x,T) = \phi^*(x)$. Therefore, 
\begin{align}
    x(t) &= x(T) - \frac{1}{\lambda} \int_T^t \nabla U(x(s),s) ds = x(T) +\frac{T-t}{\lambda}\nabla U(x(T),T) \\
    & = x(T) + \frac{T-t}{\lambda} \phi^*(x(T)). 
\end{align}
    
\end{proof}

\subsection{Wasserstein-1 proximal provides robust learning of distributions on manifolds}

 The selection  of the Wasserstein-1 proximal divergence \eqref{eq:kl:lip:divergence:primal} and \eqref{eq:kl:lip:divergence:dual} as terminal condition in \eqref{eq:our:learning:objective} relied on  two important features of this new divergence $D_f^{\Gamma_L}(P\|Q)$: (a)  it can compare mutually singular measures---a property inherited by Wasserstein-1---as can be readily seen in \eqref{eq:fgdivergence:bounds};  (b) it is smooth, i.e. has a variational derivative \eqref{eq:first:variation}, a property inherited by the KL divergence or other f-divergence. The latter is important in the well-posedness of the generative flow as captured by the MFG optimality conditions \eqref{eq:optimality:conditions1}
 and \eqref{eq:optimality:conditions2} of  \cref{thm:main}. 

\paragraph{$f$-divergences fail to learn manifolds.}
When an $f$-divergence is used as the terminal cost to train generative flows, even with a $\WII$ proximal, the mean-field game \eqref{eq:our:learning:objective} will not be well-posed. The terminal condition \eqref{eq:optimality:conditions2} is the variational derivative of $\mathcal{F}(\rho) = D_f(\rho\|\pi)$ for a fixed $\pi$, $  \frac{\delta D_f(\rho\|\pi)}{\delta \rho} = (f^\star)'\left(\frac{d\rho}{d \pi} \right)$. 
In particular, for the KL divergence 
\begin{align}
    \frac{\delta D_{KL}(\rho\|\pi)}{\delta \rho}(x) = 1+ \log \frac{d\rho}{d \pi}(x). 
\end{align}
If $\rho$ is defined as the pushforward of a standard normal distribution in $\R^d$ under invertible mapping, and the support of $\pi$ is on a lower dimensional manifold, then the terminal condition will blow up for $x \notin \text{supp}(\pi)$. In essence, the KL divergence becomes \emph{uninformative} for learning, when the target distribution is supported on a different manifold than $\rho$.

\paragraph{$\WI$ proximal of $D_f$ provides robust manifold learning}

{\rm
    First, \eqref{eq:fgdivergence:bounds}  implies that the Wasserstein-1 proximal regularized $f$-divergence ball, $\{\rho: D_{f}^{\Gamma_L}(\rho\|\pi)\le 1 \}$ contains both the $f$-divergence ball $\{\rho: D_{f}^{\Gamma_L}(\rho\|\pi)\le 1 \}$
 and the Wasserstein-1 ball $\{\rho: L \cdot W_1(\rho, \pi)\le 1 \}$. This is an important  point in our context because algorithmically a KL or an  $f$-divergence ball around a target distribution $\pi$ that is supported on a low-dimensional manifold, as is typical in many high-dimensional data sets such as images, will only include other distributions with supports  on the \emph{same} (unknown) manifold. On the other hand a ball in  $D_{f}^{\Gamma_L}(\cdot\|\pi)$
does not have this constraint and hence it is expected to be mathematically and computationally more robust for generative tasks. In particular, when learning a target distribution $\pi$  supported on an unknown manifold,  generated models $\rho_T$ can easily be singular with respect to $\pi$ unless great care is employed using carefully constructed autoencoders and/or neural architectures. Our use of the Wasserstein-1 proximal $D_{f}^{\Gamma_L}(\cdot\|\pi)$ does away with such needs, as \Cref{thm:main} suggests and as we demonstrate computationally  in \Cref{sec:numerical:examples}. This discussion also has connection to distributionally robust optimization (DRO) \cite{Rahimian2019DistributionallyRO}.
}

\paragraph{Well-defined variational derivative implies stable training}
{\rm The variational derivative $\phi^*$ in \eqref{eq:first:variation}  for the terminal condition exists and is always well-defined due to our choice of the terminal cost $\mathcal{F}$ as the Wasserstein-1 proximal of an $f$-divergence $D_f^{\Gamma_L}(\cdot \| \pi)$. Here the use of a proximal  is necessary:  (i) in the case of a pure  $f$-divergence the  variational derivative will blow up for  mutually  singular measures, while  (ii) in the case of a pure Wasserstein-1 distance, there is no uniquely defined variational derivative. 
The MFG well-posedness analysis of \Cref{thm:main} allows us to make rigorous the following heuristics. In the optimization process over possible velocities $v=v(x, t)$ in  our learning objective in \Cref{thm:main},  namely
\begin{align}
    \inf_{v, \rho} & \left\{ 
    D_f^{\Gamma_L} (\rho_T \| \pi)
     + \lambda \int_0^T \int_{\mathbb{R}^d} \frac{1}{2} |v(x,t)|^2 \rho(x,t) dx dt \right\}\, ,
    \label{eq:our:objective:end}
\end{align}
the existence of the variational derivative \eqref{eq:first:variation} of the terminal cost $D_f^{\Gamma_L} (\rho_T \| \pi)$ for any perturbations of measures including singular ones, suggests that this minimization problem is smooth in $v$ and the optimal generative flow can be  theoretically  discovered by  gradient optimization. 
We discuss and implement such algorithms inspired by \Cref{thm:main} in Section~\ref{sec:algorithm:and:indicator}. 
}

\subsection{Uniqueness of $\WI\oplus \WII$ generative flows implies well-posedness of optimization problem}

In Theorem~\ref{thm:main} we have  shown that the optimization problem \eqref{eq:our:learning:objective} and the corresponding backward-forward MFG  system \eqref{eq:optimality:conditions1} are  well-defined in the sense that the Hamiltonian $H$ exists and the terminal condition \eqref{eq:optimality:conditions2} is always well-defined.
Next we show that \eqref{eq:optimality:conditions1}
and \eqref{eq:optimality:conditions2}
have  a \textit{unique smooth solution} for a bounded domain with periodic boundary conditions, i.e., a torus.
This is a key result of this paper because it  implies that the corresponding  generative flow  with velocity field $v^*(x, t)$ is also unique and therefore can be learned in a stable manner. We will also demonstrate that latter point in our experiments in Section~\ref{sec:numerical:examples}.

\begin{theorem}[Uniqueness of Wasserstein-1/Wasserstein-2 proximal generative flows]
\label{thm:uniqueness}
    If the  backward-forward PDE system \eqref{eq:optimality:conditions1} with terminal condition
    \eqref{eq:optimality:conditions2} has smooth solutions $(U, \rho^*)$ on the torus  $\Omega$, 
    then they are unique and the solution to the optimization problem \eqref{eq:our:learning:objective} is also unique.
\end{theorem}
 We follow the strategy outlined in Theorem 2.5 of \cite{lasry2007mean}. The main idea of the proof is to consider the weak form of the continuity equation where we use the solutions of the HJ equation as test functions. We focus on the periodic domain (torus) case for the sake of simplicity, although other boundary conditions can be considered, see \cite{lasry2007mean}. The proof requires two properties: (a) the convexity of the Hamiltonian $H(x,p) = \frac{1}{2\lambda}|p|^2$; and (b) the terminal conditions \eqref{eq:optimality:conditions2} is necessarily increasing in some suitable sense in $\rho$. Convexity of the Hamiltonian is clear. As for the mononicity requirement, we first prove the following Lemma \ref{lemma:monotonicity}. The result is based on the properties of the $\WI$-proximal. 

\begin{lemma}[Monotonicity of \eqref{eq:optimality:conditions2}]
\label{lemma:monotonicity}
    If $\phi_i^*$ is the optimizer \eqref{eq:optimality:conditions2} of the dual problem \eqref{eq:kl:lip:divergence:dual}
    for $D_f(\rho_i(\cdot, T) \| \pi)$ for $i=1,2$, then
    \begin{equation}
    \label{eq:monotonicity}
        \int_\Omega (\phi_1^*(x) - \phi_2^*(x)) d(\rho_1(\cdot, T) - \rho_2(\cdot, T))(x) \geq 0
    \end{equation}
    for all $\rho_1(\cdot, T), \rho_2(\cdot, T) \in \mathcal{P}(\Omega)$.
\end{lemma}
\begin{proof}
   For the simplicity we write $\mu_1 = \rho_1(\cdot,T)$ and $ \mu_2=\rho_2(\cdot,T)$.
    We rewrite \eqref{eq:monotonicity} and use the optimality of $\phi_i^*$ for the dual formulation \eqref{eq:kl:lip:divergence:dual} of $D_f^{\Gamma_L}(\mu_i \| \pi)$ with $i=1,2$ to get
    \begin{align}
    \begin{split}
        & \mathbb{E}_{\mu_1}[\phi_1^*] - \mathbb{E}_{\mu_2}[\phi_1^*] - \mathbb{E}_{\mu_1}[\phi_2^*]  + \mathbb{E}_{\mu_2}[\phi_2^*]  \\
        = & \{D_f^{\Gamma_L}(\mu_1 \| \pi) + \mathbb{E}_\pi [f^\star(\phi_1^*)]\} - \mathbb{E}_{\mu_2}[\phi_1^*] - \mathbb{E}_{\mu_1} [\phi_2^*] + \{\mathbb{E}_\pi [f^\star(\phi_2^*)] + D_f^{\Gamma_L}(\mu_2 \| \pi) \}  \\
        \geq & D_f^{\Gamma_L}(\mu_1 \| \pi) - D_f^{\Gamma_L}(\mu_2 \| \pi) - D_f^{\Gamma_L}(\mu_1 \| \pi) + D_f^{\Gamma_L}(\mu_2 \| \pi)   \geq 0. 
    \end{split}
    \end{align}

\end{proof}

We now prove Theorem~\ref{thm:uniqueness}. 

\begin{proof}

Let $(U_1,\rho_1)$ and $(U_2,\rho_2)$ be two smooth solutions of \eqref{eq:optimality:conditions1} and \eqref{eq:optimality:conditions2}. Define $U  = U_1 - U_2$ and $\rho = \rho_1-\rho_2$. We prove $U = 0$ and $\rho = 0$. Based on \Cref{lemma:monotonicity}, 
we have the  monotonicity of the terminal condition $U(\cdot,T)$. Furthermore,  $H(x, p)=\frac{1}{2\lambda}|p|^2$, is convex with respect to the last variable. Without loss of generality, let $\lambda = 1$. Observe that $\nabla_p^2 H(x,p) = I_d$. Now observe that
    \begin{align}
    \label{eq:proof:uniqueness:eq1}
        \begin{split}
        &\frac{d}{dt} \int_\Omega U \rho dx = \int_\Omega [(\partial_t U) \rho + U(\partial_t \rho)] dx \\
        = & \int_\Omega \left(H(x, \nabla U_1) - H(x, \nabla U_2) \right) \rho dx  \\
       + & \int_\Omega U \left(\nabla \cdot (\rho_1 \nabla_p H(x, \nabla U_1)) - \nabla \cdot (\rho_2 \nabla_p H(x, \nabla U_2))\right) dx.
       \end{split}
    \end{align}
Via integration by parts, with $\langle\cdot, \cdot\rangle$ denoting the standard inner product  between vectors in $\mathbb{R}^d$,  we get
    \begin{align}
    \label{eq:proof:uniqueness:eq2}
    \begin{split}
       \eqref{eq:proof:uniqueness:eq1}  = & \int_\Omega \left[(H(x, \nabla U_1) - H(x, \nabla U_2) ) \rho  -  \langle\nabla U, \rho_1 \nabla_p H(x,\nabla U_1)- \rho_2 \nabla_p H(x,\nabla U_2)\rangle \right] dx \\
       =  & - \int_\Omega \rho_1 \left[H(x, \nabla U_2) - H(x, \nabla U_1) )  -  \langle\nabla U_2 - \nabla U_1, \nabla_p H(x,\nabla U_1)\rangle \right] dx \\
       - & \int_\Omega \rho_2 \left[H(x, \nabla U_1) - H(x, \nabla U_2) )   -  \langle\nabla U_1 - \nabla U_2, \nabla_p H(x,\nabla U_2)\rangle \right] dx. 
       \end{split}
    \end{align}
    Due to the convexity of $H$, a Taylor series expansion yields
    \begin{equation}
        \eqref{eq:proof:uniqueness:eq2} = \frac{d}{dt} \int_\Omega U \rho dx \leq - \int_\Omega \frac{\rho_1 + \rho_2}{2}|\nabla U_1 - \nabla U_2|^2 dx \leq 0.
    \end{equation}
    By integrating this inequality on the time interval $[0, T]$ we obtain
    \begin{equation}
    \label{eq:proof:uniqueness:eq3}
        \int_\Omega U(x, T) \rho(x, T) dx \leq \int_\Omega U(x, 0) \rho(x, 0) dx - \int_0^T \int_\Omega \frac{\rho_1 + \rho_2}{2}|\nabla U_1 - \nabla U_2|^2 dx dt
    \end{equation}
    Using that $\rho(x, 0)=0$  and the monotonicity \eqref{eq:monotonicity}, we have the LHS of \eqref{eq:proof:uniqueness:eq3} is nonnegative.
    It implies that $\int_\Omega U(x, T) \rho(x, T) dx = 0$ resulting in
    \begin{equation}
        \nabla U_1 = \nabla U_2  \text{ in } \{\rho_1 >0\} \cup \{\rho_2 >0\}.
    \end{equation}
    This proves that $\rho_1$ solves the same equation as $\rho_2$, with the same drift $\nabla_p H(x, \nabla U_1)=\nabla_p H(x, \nabla U_2)$ and therefore $\rho_1 = \rho_2$. Correspondingly $U_1$ and $U_2$ solve the same HJ equation and therefore $U_1 = U_2$, as they share the same terminal condition \eqref{eq:optimality:conditions2}.
\end{proof}

From a computational perspective, this theorem provides confidence that any numerical implementation to the optimization problem \eqref{eq:our:objective} is approximating a single unique limiting solution. Algorithmically, this means that the training will not oscillate between multiple modes representing different possible limiting flows.

\begin{remark}\label{rmk:classicalsolution}
    {\rm In \Cref{thm:uniqueness} we assumed the existence of classical solutions to prove uniqueness. However first order HJ do not always have classical solutions, hence further investigation is needed in future work. For instance, we intend  to consider \emph{stochastic} normalizing  flows for which the corresponding MFG has classical solutions \cite{lasry2007mean},  and a uniqueness result in the spirit of \Cref{thm:uniqueness} can hold without additional regularity assumptions.}
\end{remark}

\section{Proximal regularization produces adversarial training algorithm}

\label{sec:algorithm:and:indicator}

In this section, we provide more details as to how the objective function for training $\WI\oplus \WII$ proximal generative flows are computed. The mean-field game formulation in \eqref{eq:our:learning:objective} enables implementation of the new flow through \emph{adversarial} training algorithms. This is due to the dual formulation for computing the $D_f^\Gamma$ divergence, and has been successfully applied in generative modeling tasks through $D_f^\Gamma$-based generative adversarial networks \cite{birrell2020f,birrell2022structure,birrell2023functionspace}, and Wassserstein gradient flows of $D_f^\Gamma$ \cite{gu2023GPA}. A generative adversarial flow has previously been proposed in \cite{grover2018flowGAN}. Our formulation, however, resolves the persistent ill-posedness issues. 

\subsection{An adversarial training algorithm for $\WI\oplus\WII$ generative flows
}

We compute the objective function in \eqref{eq:our:learning:objective}, see also \eqref{eq:our:objective}, and we will see that in contrast to normalizing flows and its $\WII$ proximal regularizations \cite{grathwohl2018ffjord,finlay2020train,onken2021ot}, we only need to simulate forward trajectories of the flow during training. Suppose we are given a dataset $\{X^{(n)} \}_{n = 1}^N \sim \pi$. Let the discriminator $\phi(x;\theta_\phi)$ and potential function $U(x,t;\theta_U)$ be parametrized via neural networks, with parameters $\theta_\phi$ and $\theta_U$, respectively. Recall that $\phi$ needs to be $L$-Lipschitz continuous. Fix an initial reference distribution $\rho_0$ (e.g., a normal distribution), and simulate $M$ independent trajectories with initial condition $\{Y_0^{(m)}\}_{m = 1}^M\sim \rho_0$ and velocity field $v(y,t) = -\frac{1}{\lambda} \nabla_x U(y(t),t)$. The number of simulated trajectories $M$ do not need to match the number of samples $N$ from $\pi$. We follow a discretize-then-optimize approach for training the model parameters, as described in \cite{onken2021ot}. Each trajectory is simulated using a forward Euler scheme. Discretize the time interval $[0,T]$ into $K$ equal segments of length $h$ so that $t_K = T = Kh$. Then for the $m$th trajectory at time $t_k = kh$, we have the recurrence relation
\begin{align}
\label{eq:recurrence:relation}
    Y_{k+1}^{(m)}  = Y_k^{(m)} - \frac{h}{\lambda}\nabla U(Y_k^{(m)},kh) \sim \rho(\cdot,(k+1)h), \,\,\,\, \text{ for } k = 0,\ldots, K-1. 
\end{align}
We first estimate the divergence $D_f^\Gamma(\rho(\cdot,T)\|\pi)$ by finding the optimal $\phi \in \Gamma_L$
\begin{align}\label{eq:discriminator:optimization}
    &\max_{\theta_\phi} \mathcal{J}_\phi(\theta_\phi) =\max_{\theta_\phi} \left\{ \frac{1}{M}\sum_{m=1}^M \phi\left(Y^{(m)}_{K}; \theta_\phi\right) - \frac{1}{N}\sum_{n=1}^N f^\star\left(\phi(X^{(n)}; \theta_\phi)\right) + \mathcal{L}_\phi(\theta_\phi)\right\} \\
    & \text{where } ~~ \mathcal{L}_\phi(\theta_\phi) =  - \sum_{n=1}^{\min (M,N)} \max\left(\left|\nabla \phi (c_n Y_K^{(n)} + (1-c_n)X^{(n)}; \theta_\phi)\right|^2 - L^2, 0\right).
    \label{eq:gp}
\end{align}
Here, $\mathcal{L}_\phi(\theta_\phi)$ penalizes discriminators that have Lipschitz constant greater than $L$ for random convex combinations of $Y_K^{(n)}$ and $X^{(n)}$ \cite{birrell2020f,gulrajani2017improvedtrainingwassersteingans}. Here, $c_n$ is randomly sampled from a uniform distribution on $[0,1]$. Discrminator parameters $\theta_\phi$ are updated according to a gradient ascent method $\theta_\phi^{(l+1)} = \theta_\phi^{(l)} + \eta\nabla_{\theta_\phi} \mathcal{J}_\phi(\theta_\phi^{(l)})$ for some learning rate $\eta$. Let $\phi^*(x) = \phi(x; \theta^*_\phi)$ the resulting optimal discriminator. Thereafter, the objective function $\mathcal{J}(\theta_U)$ for learning the parameters of $U$ is estimated by
\begin{align}
      \mathcal{J}_U(\theta_U) =  \underbrace{\frac{1}{M}\sum_{m=1}^M \phi^*\left(Y^{(m)}_{K}\right) - \frac{1}{N}\sum_{n=1}^N f^\star\left(\phi^*(X^{(n)})\right) }_{\text{Estimates } D_f^{\Gamma_L}(\rho(\cdot,T)\|\pi)} + \underbrace{\frac{\lambda h}{2M} \sum_{m = 1}^M \sum_{k = 0}^{K-1} \left|\frac{\nabla U(Y_k^{(m)},kh; \theta_U)}{\lambda} \right|^2}_{\text{Estimates } \lambda \int_0^T \int_{\R^d} \frac{1}{2}|v(x,t)|^2 \rho(x,t) dxdt}.
\end{align}
Model parameters $\theta_U$ are similarly updated via gradient descent:  $\theta_U^{(l+1)} = \theta_U^{(l)} - \eta\nabla_{\theta_U} \mathcal{J}_U(\theta_U^{(l)})$. From here, new sample trajectories are simulated using the potential at step $l+1$, $U^{(l+1)}(x,t) = U(x,t; \theta_U^{(l+1)})$ and the parameters are updated again. Like GANs, there is a minimax structure in the optimization problem, where the inner optimization problem is solved for each update of the outer problem. Our proposed generative model can be viewed as an \emph{adversarial generative flow} given in terms of a minimax problem for a discriminator $\phi^*$ and a potential function $U^*$. The pseudocode for learning $U^*$ and $\phi^*$ are described in Algorithms~\ref{alg:wasserstein:proximal:generative:flow} and \ref{alg:wasserstein:proximal:discriminator}, respectively.

\begin{algorithm}
\caption{Training $\WI\oplus \WII$ proximal generative flows}\label{alg:wasserstein:proximal:generative:flow}
\begin{algorithmic}[1]

\State \textbf{Input:} Dataset $\{X^{(n)} \}_{n = 1}^N \sim \pi$

\State \textbf{Initialize:} Discriminator $\phi(x;\theta_\phi)$ and potential function $U(x,t;\theta_U)$ with parameters $\theta_\phi$ and $\theta_U$, Lipschitz constant $L$, initial reference distribution $\rho_0$, number of trajectories $M$, time interval $[0,T]$, with step size $h$, $T = Kh$, learning rate $\eta$, $N_{iter}^\phi$, $N_{iter}^U$

\vspace{0.2cm}

\For{$l \gets 1$ to $N^U_{iter}$}

    \State Sample $M$ initial conditions $\{Y_0^{(m)}\}_{m = 1}^M \sim \rho_0$
    \For{$m = 1$ to $M$} \Comment{Simulate trajectories}
        \For{$k = 0$ to $K-1$}
            \State $Y_{k+1}^{(m)} = Y_k^{(m)} - \frac{h}{\lambda}\nabla U(Y_k^{(m)}, kh; \theta_U)$
        \EndFor

    \EndFor

    \vspace{0.2cm}
    \State $\phi^*(x) \gets $ \textsc{LearnDiscriminator}$(\{X^{(n)}\},\{Y_K^{(m)}\},L,\eta,N_{iter}^\phi)$

    \vspace{0.2cm}
    \State $\mathcal{D}_U(\theta_U) \gets \frac{1}{M} \sum_{m=1}^M \phi^*(Y^{(m)}_K) - \frac{1}{N} \sum_{n=1}^N f^\star(\phi^*(X^{(n)}))$ \Comment{Estimate $D_f^{\Gamma_L}(\rho_T\|\pi)$ }

    \vspace{0.2cm}
    \State $\mathcal{J}_U(\theta_U) \gets \mathcal{D}_U(\theta_U) + \frac{\lambda h}{2M} \sum_{m = 1}^M \sum_{k = 0}^{K-1} \left|\frac{\nabla U(Y_k^{(m)}, kh; \theta_U)}{\lambda} \right|^2$ \Comment{Estimate objective}

    \vspace{0.2cm}
    \State $\theta_U \gets \theta_U - \eta \nabla_{\theta_U} \mathcal{J}_U(\theta_U)$ \Comment{Update $\theta_U$ via gradient descent}

    \vspace{0.2cm}
    \State $\mathcal{R}^{(l)}_{HJ} \gets \frac{h}{M} \sum_{m = 1}^M \sum_{k = 0}^{K-1}\left| \partial_t U(Y_k^{(m)},hk;\theta_U) + \frac{1}{2\lambda}\left| \nabla  U(Y_k^{(m)},hk;\theta_U)\right|^2\right| $ 

    \vspace{0.2cm}
    \State $\mathcal{R}_{T}^{(l)} \gets \frac{1}{M}\sum_{m = 1}^M \left|\nabla U(Y_K^{(m)},T; \theta_U) - \nabla \phi^*(Y_K^{(m)}) \right|$ \Comment{Optimality indicators \eqref{eq:indicator1}, \eqref{eq:indicator2}}

    \vspace{0.2cm}
\EndFor

\end{algorithmic}
\end{algorithm}

\begin{algorithm}
\caption{Learning the discriminator $\phi^*(x)$}\label{alg:wasserstein:proximal:discriminator}
\begin{algorithmic}[1]
\Function{LearnDiscriminator}{$\{X^{(n)} \}_{n = 1}^N, \{Y_K^{(m)}\}_{m = 1}^M, L, \eta, N_{iter}^\phi$}
    \State \textbf{Initialize:} Discriminator $\phi(x;\theta_\phi)$ with parameters $\theta_\phi$
    \For{$l \gets 1$ to $N_{iter}^\phi$}
        \vspace{0.2cm}

        \State $c_n \sim \mathcal{U}[0,1]$ for $n = 1, \ldots, \min(M,N)$
            \vspace{0.2cm}

        \State $Z^{(n)} \gets c_n Y_K^{(n)} + (1-c_n)X^{(n)}$ for $n = 1, \ldots, \min(M,N)$
            \vspace{0.2cm}

        \State $\mathcal{L}_\phi(\theta_\phi) \gets - \sum_{n=1}^{\min(M,N)} \max\left(\left|\nabla \phi (Z^{(n)}; \theta_\phi)\right|^2 - L^2, 0\right)$ \Comment{Lipschitz regularization}
            \vspace{0.2cm}

        \State $\mathcal{D}_\phi(\theta_\phi) \gets \frac{1}{M} \sum_{m=1}^M \phi(Y^{(m)}_K; \theta_\phi) - \frac{1}{N} \sum_{n=1}^N f^\star(\phi(X^{(n)}; \theta_\phi))$
            \vspace{0.2cm}

        \State $\mathcal{J}_\phi(\theta_\phi) \gets \mathcal{D}_\phi(\theta_\phi) + \mathcal{L}_\phi(\theta_\phi)$ 
            \vspace{0.2cm}

        \State $\theta_\phi \gets \theta_\phi + \eta \nabla_{\theta_\phi} \mathcal{J}_\phi(\theta_\phi)$ \Comment{Update $\theta_\phi$ via gradient ascent}
            \vspace{0.2cm}

    \EndFor
    \State \textbf{return} $\phi^*(x) = \phi(x; \theta_\phi)$
\EndFunction
\end{algorithmic}
\end{algorithm}

\subsection{Hamilton-Jacobi optimality indicators} 
The optimality conditions 
\eqref{eq:optimality:conditions1} and \eqref{eq:optimality:conditions2}
of Theorem \ref{thm:main} provide \emph{a posteriori} error estimators that may provide insight into the quality of the learned generative flow. Recall that the Hamilton-Jacobi equation \eqref{eq:indicator1} and its terminal condition  \eqref{eq:indicator2}  characterize the optimal velocity field of the generative flow, which means at the end of training, the learned flow should satisfy the HJ equation. 

We define the \emph{Hamilton-Jacobi residual} and the \emph{Hamilton-Jacobi terminal condition error} as follows: 
\begin{align}
\label{eq:indicator1}
   & \text{HJ residual}:  \mathcal{R}_{HJ}(\theta_U) = \frac{h}{M} \sum_{m = 1}^M \sum_{k = 0}^{K-1}\left| \partial_t U(Y_k^{(m)},hk;\theta_U) + \frac{1}{2\lambda}\left| \nabla  U(Y_k^{(m)},hk;\theta_U)\right|^2\right|\\
 & \text{HJ terminal condition error: }  \mathcal{R}_{T}(\theta_U) = \frac{1}{M}\sum_{m = 1}^M \left|\nabla U(Y_K^{(m)},T;\theta_U) - \nabla \phi^*(Y_K^{(m)}) \right|.
 \label{eq:indicator2}
\end{align}
These error functions are not used for training, but rather are used as \emph{a posteriori} indicators of optimality of the learned generative flow. This is in constrast to \cite{onken2021ot}, where the Hamilton-Jacobi equation is used as an additional regularizer during training. Due to  Theorem~\ref{thm:uniqueness}, there is a \textit{unique} (smooth) optimal generative flow and is expected to satisfy the optimality conditions \eqref{eq:optimality:conditions1} and \eqref{eq:optimality:conditions2}. Hence any computed generative flow
based on \Cref{alg:wasserstein:proximal:generative:flow} can be tested against  the corresponding  HJ residuals \eqref{eq:indicator1} and  \eqref{eq:indicator2}. In this sense they can also be used as an optimality indicator for \emph{real-time} termination of our simulation when the HJ residuals are sufficiently small. In \Cref{sec:numerical:examples}, we successfully demonstrate the algorithm on a MNIST dataset that avoids the use of autoencoders. We show that the optimality condition are eventually satisfied in the course of our training. Additional computational costs for calculating \eqref{eq:indicator1} and \eqref{eq:indicator2} take $\mathcal{O}(K M)$ and $\mathcal{O}(Md)$ gradient evaluations, respectively, where $M, K$ are defined in \eqref{eq:indicator1} and  \eqref{eq:indicator2}.

\subsection{Adversarial training eliminates the need for forward-backward simulation}
\label{rmk:algorithm:adversarial}
Continuous Normalizing Flow (CNF) models, such as OT flow \cite{onken2021ot} requires both forward and backward simulation of invertible flows during training. These models learn the flows in the \textit{forward} normalizing direction 
\begin{equation}\label{eq:CNF}
    \inf_{v, \rho} D_{KL} (\rho(\cdot, T), \mathcal{N}_d) + \lambda \int_0^T \int_{\mathbb{R}^d} \frac{1}{2}|v(x,t)|^2 \rho(x,t) dxdt
\end{equation}
where $\mathcal{N}_d$ refers to the $d$-dimensional Gaussian distribution, 
and generate samples using the \textit{backward} flows $\gamma_t - \nabla \cdot (\gamma v) = 0$ for $\gamma(\cdot, t) = \rho(\cdot, T-t)$. 
A main reason for the forward-backward simulation is because the likelihood needs to be computed by integrating over the path to estimate the value of the KL divergence. 
However, in high-dimensional examples, the data $\pi = \rho_0$  often lie on low dimensional manifolds and immediately two challenges arise.

First, due to the lack of absolute continuity, a KL-ball in \eqref{eq:CNF} fails to contain an optimizer, leading to unstable training and mathematically to an undefined terminal condition for the HJ equation $U(\cdot, T)$ if we use $D_{KL}$ as the terminal condition instead of the Wassersten-1 proximal in \Cref{thm:main}. 
A second mathematical and algorithmic reason for the failure of CNFs in the case where the target is supported on a lower dimensional manifold: there is no \textit{invertible} mapping between the Gaussian and a target distribution $\pi$ supported on a lower-dimensional distribution, leading to non-existence of an  inverse transport map  from $\mathcal{N}_d$ to $\rho_0 = \pi$. For this reason, autoencoders are necessary for OT flow \cite{onken2021ot} and potential flow generators \cite{yang2019potential} to embed the flows into the appropriate spaces before applying the generative models.

Unlike CNF-based models such as \eqref{eq:CNF}, our model in \Cref{alg:wasserstein:proximal:generative:flow}, based on \Cref{thm:main}, bypasses both such challenges through: (a)  \emph{adversarial training}, which does not require inversion of the flow,  and (b) \emph{Wasserstein-1 proximal regularized $f$-divergences}, which do not require absolute continuity between measures. Therefore our model is able to generate images without the aid of autoencoders, as we demonstrate next.

\section{Numerical results}
\label{sec:numerical:examples}
We demonstrate that our proposed  generative flows \eqref{eq:our:learning:objective} formulated combining Wasserstein-1 and Wasserstein-2 proximal regularizations facilitate learning distributions that are supported on low-dimensional manifolds. 
We learn the MNIST dataset on $\Omega = [0,1]^{784}$ using the adversarial flow model described in \Cref{alg:wasserstein:proximal:generative:flow}, without the use of any autoencoders. 
To demonstrate the impact of Wasserstein-1 and Wasserstein-2 proximal regularizations, we consider the following four test cases: 
\begin{enumerate}
    \item Unregularized flow: no running cost penalty for $t\in (0,T)$ and the terminal cost is the standard dual formulation of $f$-divergences

    \begin{equation}
       \inf_{v} \left\{ \sup_{\phi \in C(\mathbb{R}^d)} \left\{\mathbb{E}_{\rho(\cdot, T)}[\phi] - \mathbb{E}_\pi [f^\star(\phi)] \right\} : \frac{dx}{dt} = v(x(t),t), \, x(0) \sim \rho(\cdot,0) \right\}
    \end{equation}

    \item Wasserstein-2 proximal of $f$-divergence flow: optimal transport cost for $t \in (0,T)$ and the terminal cost is the standard dual formulation of $f$-divergences

    \begin{align}
        &\inf_{v}\left\{\sup_{\phi \in C(\mathbb{R}^d)} \left\{\mathbb{E}_{\rho(\cdot, T)}[\phi] - \mathbb{E}_\pi [f^\star(\phi)] \right\} + \lambda \int_0^T \int_{\mathbb{R}^d} \frac{1}{2}|v(x,t)|^2 \rho(x,t) dx dt \right\} \\
       & \text{s.t. }\frac{dx}{dt} = v(x(t),t), \, x(0) \sim \rho(\cdot,0) \nonumber
    \end{align}
    
    \item Wasserstein-1 proximal flow of $f$-divergence flow: no running cost penalty for $t\in (0,T)$ and the terminal cost is the regularized dual formulation of $f$-divergences  
    \begin{equation}
    \label{eq:case:w1-proximal:flow}
        \inf_v\left\{\sup_{\phi \in \Gamma_L} \left\{\mathbb{E}_{\rho(\cdot, T)}[\phi] - \mathbb{E}_\pi [f^\star(\phi)] \right\}: \frac{dx}{dt} = v(x(t),t), \, x(0) \sim \rho(\cdot,0) \right\}
    \end{equation}
    
    \item Composed Wasserstein-1/Wasserstein-2 proximal flow (W1/W2 proximal flow) : optimal transport cost for $t \in (0,T)$ and the terminal cost is the regularized $f$-divergence 
    \begin{align}
        \label{eq:case:w1w2-proximal:flow}
        &\inf_{v}\left\{\sup_{\phi \in \Gamma_L} \left\{\mathbb{E}_{\rho(\cdot, T)}[\phi] - \mathbb{E}_\pi [f^\star(\phi)] \right\} + \lambda \int_0^T \int_{\mathbb{R}^d} \frac{1}{2}|v(x,t)|^2 \rho(x,t) dx dt \right\} \\
       & \text{s.t. }\frac{dx}{dt} = v(x(t),t), \, x(0) \sim \rho(\cdot,0) \nonumber
    \end{align}
\end{enumerate}
For our numerical experiments, we choose the Lipschitz constant to be $L = 1$, the weighting parameters to be $\lambda = 0.05$, the terminal time $T = 5$, time step size $h = 1.0$, and $f(x)=-\log(x)$ for the reverse KL divergence. In addition, the discriminator $\phi$ and the potential $U$ share a common neural network architecture; see \Cref{appendix:subsec:nn:architecture}.

\begin{figure}[h]
    \centering
    \begin{subfigure}{.4\linewidth}
        \includegraphics[width=\linewidth]{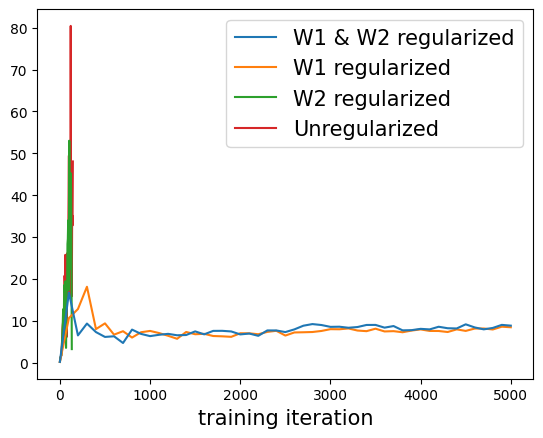}
        \caption{Terminal cost $D_f^{\Gamma_L}(\rho(\cdot,T) \| \pi)$ \eqref{eq:kl:lip:divergence:dual}}
        \label{subfig:terminal:cost}
    \end{subfigure} 
     \begin{subfigure}{.41\linewidth}
     \includegraphics[width=\linewidth]{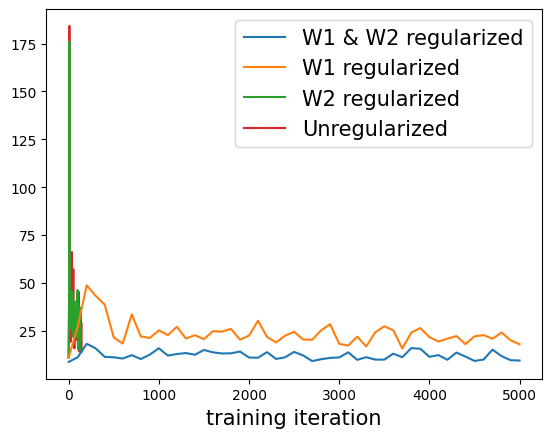}
    \caption{Kinetic energy}
    \label{subfig:kinetic:energy}
    \end{subfigure}
    \begin{subfigure}{.41\linewidth}
        \includegraphics[width=\linewidth]{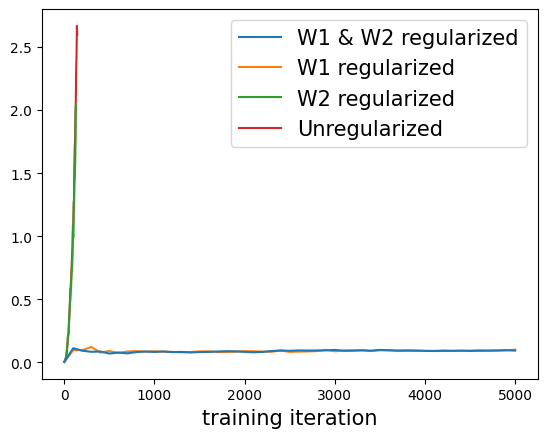}
      \caption{Optimality indicator \eqref{eq:indicator2}}
      \label{subfig:hjb:condition:error}
    \end{subfigure}
    \begin{subfigure}{.4\linewidth}
    \includegraphics[width=\linewidth]{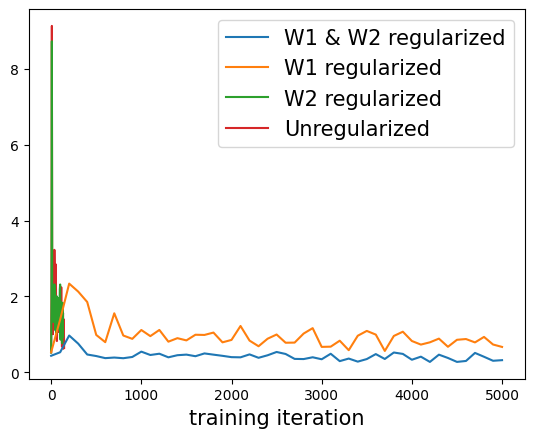}
    \caption{Optimality indicator \eqref{eq:indicator1}}
    \label{subfig:hjb:residual}
    \end{subfigure}
    
    \caption{Evaluation of learning objectives (a - b) and optimality indicators (c - d) from MFG optimality conditions over the course of training. We learn the MNIST dataset using the four test cases with  $T=5.0$, $h=1.0$.  We choose $\lambda=0.05$ and $L=1$ for the weight paramters of Wasserstein-2 and Wasserstein-1 proximals, respectively. In Figure $\bf{(a)}$, observe that the terminal cost  $\mathcal{F}(\rho(\cdot,T)) = D_f^{\Gamma_L}(\rho(\cdot,T) \| \pi)$  diverges (green, red) without the Wasserstein-1 proximal regularization.  
    {\bf (b)} The Wasserstein-2 proximal additionally regularizes the flow to have lower kinetic energy and to be less oscillatory training objectives. Less oscillation is also related to the uniqueness of the MFG solution in \Cref{thm:uniqueness} which in turn is expected to render the algorithms more robust, i.e. in our context  not susceptible to implementation-dependent choices.
    {\bf (c)} As inferred from mean-field game, the learning problem without Wasserstein-1 proximal regularization lacks a well-defined terminal condition.  The exploding optimality indicator \eqref{eq:indicator2} exemplifies this behavior. {\bf (d)} $\WI\oplus\WII$ proximal generative flow results in lower values of the optimality indicator \eqref{eq:indicator1} compared to those from  Wasserstein-1 proximal regularized flow. {\bf (c-d)} show that the optimality indicators can inform when an optimal generative flow has been discovered.     }
    \label{fig:mnist}
\end{figure}

\begin{figure}[h]
    \centering
    \begin{subfigure}{\linewidth}
     \includegraphics[width=.32\linewidth]{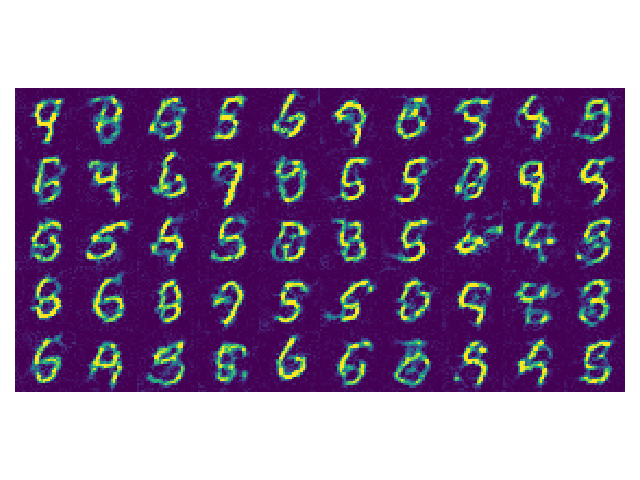}
     \includegraphics[width=.32\linewidth]{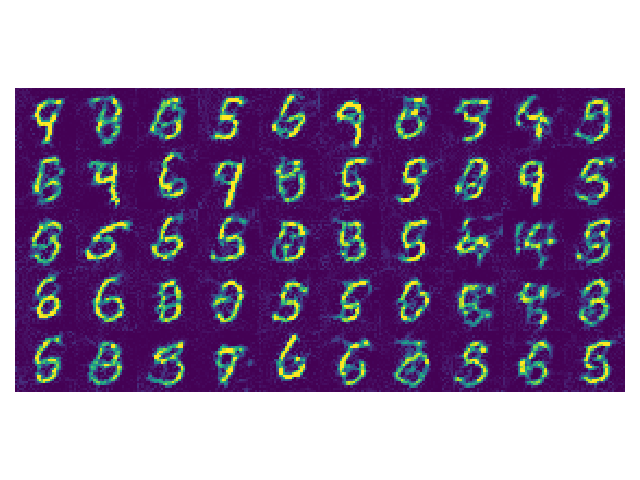}
     \includegraphics[width=.32\linewidth]{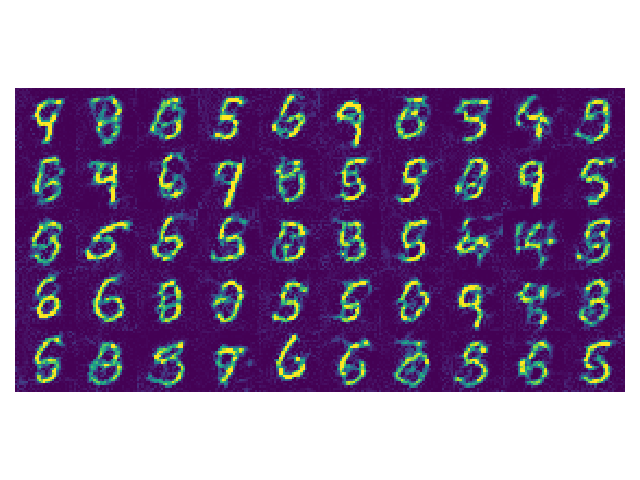}
    \caption{Generated samples from $\WI\oplus\WII$ proximal generative flow \eqref{eq:case:w1w2-proximal:flow} with different time step sizes $h = 2^0$ (left), $h = 2^{-3}$ (center), $h = 2^{-6}$ (right)}
    \label{subfig:w1w2:generated:samples}
    \end{subfigure}
    
    \begin{subfigure}{\linewidth}
     \includegraphics[width=.32\linewidth]{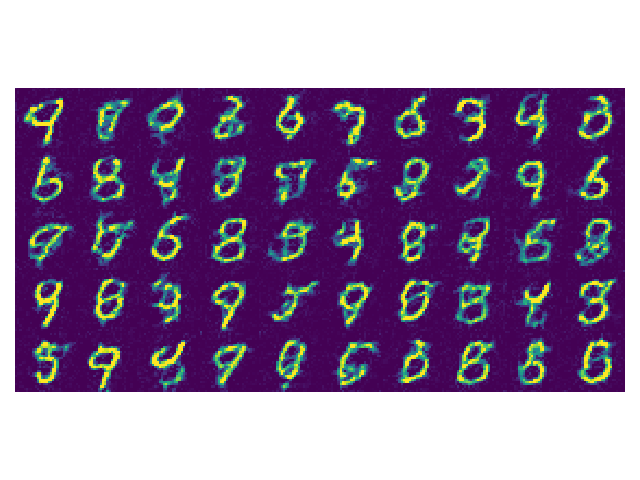}
     \includegraphics[width=.32\linewidth]{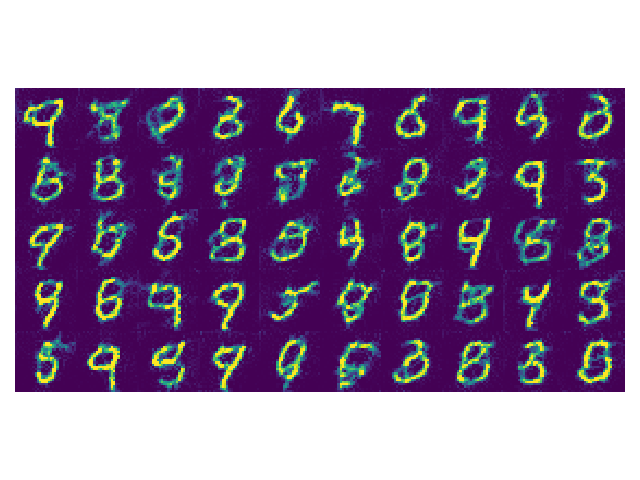}
     \includegraphics[width=.32\linewidth]{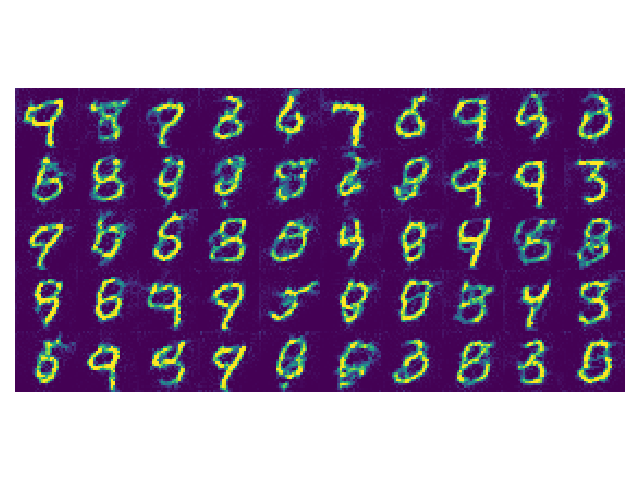}
    \caption{Generated samples from $\WI$ proximal flow \eqref{eq:case:w1-proximal:flow} with different time step sizes $h = 2^0$ (left), $h = 2^{-3}$ (center), $h = 2^{-6}$ (right)}
    \label{subfig:w1w2:generated:samples}
    \end{subfigure}
    \caption{ Wasserstein-2 proximal regularization implies discretization invariance in generative flows. After learning the MNIST dataset $\WI\oplus\WII$  and $\WI$ proximal generative flows with  $T=5.0$, $h=1.0$, we generated samples by integrating the learned vector field $\dot{x}(t) = -\frac{1}{\lambda}\nabla U(x(t), t)$ over time with different time step sizes $h$. 
    In {\bf (a)}, we see that Wasserstein-2 proximal regularization provides almost straight flow trajectories which leads to generated samples which are almost invariant to time discretization. This robustness of the continuous time flow generator ensures high fidelity of generated samples regardless of time discretization. This empirical observation is justified by the theoretical result in 
    \eqref{eq:trajectories}  of Theorem~\ref{thm:main}. On the other hand, we see in {\bf (b)} that without Wasserstein-2 regularization, the resulting vector field is more sensitive to varying step sizes as certain digits may flip to other ones.
    }
    \label{fig:mnist2}
\end{figure}

\Cref{fig:mnist} demonstrates the indispensability of using the Wasserstein-1 proximal regularization for stabilizing the training procedure. Observe that the flows that are not regularized or only rely on Wasserstein-2 regularization \emph{fail} in the middle of their training procedures as their terminal cost blows up soon after initialization (see red and green curves in \Cref{subfig:terminal:cost}). This is because the terminal cost cannot effectively compare measures $\rho(\cdot,T)$ (which is absolutely continuous with respect to $\rho(\cdot,0)$ and $\pi$. The Wasserstein-2 proximal does not avoid this issue as it alone cannot make up for an ill-defined terminal condition.

Moreover, we see in \Cref{subfig:kinetic:energy} that for the flow trained on only the Wasserstein-1 proximal, the resulting flow has substantially higher kinetic energy than the the flow that uses both proximals. This implies that the trajectories are not as regular or potentially strongly nonlinear. Moreover, it is likely that the learned flow is strongly dependent on the particular parametrization of the velocity fields.

Our continuous-time flow model stands out for its ability to learn high-dimensional problems directly in the original space, without 
specialized architecture or the use of autoencoders as in normalizng flows. OT-flow \cite{onken2021ot} and the potential-flow generator \cite{yang2019potential} applies autoencoders to get latent representations of the target in lower dimensional spaces and learn normalizing flows in the lower dimensional space. Additionally, our flow can be learned using a discretize-then-optimize (DTO) approach with just a few time steps $K = 5$. We are able to use such a crude step size as the optimal trajectories are known to be provably linear (Theorem \ref{thm:main}). This demonstrates that our problem formulation \eqref{eq:our:learning:objective} enables more efficient training compared to methods with highly engineered architectures such as the 10-layer-stacked CNF setting used for RNODE \cite{finlay2020train}. Our approach is similar to the GAN implementation of the Potential flow generator \cite{yang2019potential}. But our algorithm differs from  it in the sense that we solve the Wasserstein-2 regularized problem with a Wasserstein-1 regularized terminal cost whose combination provides well-defined optimality conditions \eqref{eq:optimality:conditions1}, \eqref{eq:optimality:conditions2} without imposing them as regularization terms in the objective function.

In Figure \ref{fig:mnist2}, we can see that when the two proximals are combined, the resulting generative flows are \emph{discretization-invariant} as the generated images do not change with the step size. This further supporting our claim that a well-posed generative does not have parametrization-dependent performance. This behavior is due to Theorem~\ref{thm:main}, which show that the optimal trajectories are linear, so there is no increased discretization errors with larger step sizes. On the other hand, when only the $\WI$ proximal is applied, we see that some images may change with different step sizes, meaning that the trajectories are incurring extra discretization error with larger step sizes as they are not linear. 

\section{
Discussion and conclusion}
\label{sec:limitation:conclusion}

In this paper, we introduced $\WI\oplus\WII$ proximal generative flows, which combines Wasserstein-1 and Wasserstein-2 proximal regularizations for stabilizing the training of continuous-time deterministic normalizing flows, especially for learning high-dimensional distributions supported on low-dimensional manifolds. The $\WII$ proximal regularization enforce flows that minimize kinetic energy, resulting in straight line paths that are simpler to learn. 
We showed in Theorem~\ref{thm:main} that the optimal trajectories are exactly linear, and we empirically demonstrated this fact in Figure~\ref{fig:mnist2} by showing the trained flow is discretization invariant.

For learning distributions on  manifolds, we choose the terminal cost for the generative flow to be Wasserstein-1 proximal regularization of $f$-divergences, which can compare mutually singular measures. Typically targets are supported on an unknown data manifold, so generated models can easily miss the manifold and be singular with respect to the target distribution.
In addition, Wasserstein-1 proximal regularizations of $f$-divergences inherit from $f$-divergences the 
existence of variational derivatives  for any perturbations of measures,
including singular ones, whereas both $f$-divergences and the Wasserstein-1 distance fail in that respect. This differentiability property suggests that corresponding minimization problems are smooth and the optimal
generative flow can be robustly discovered by gradient optimization. In the numerical example, we saw that without a $\WI$ proximal, the training objective would diverge quickly.

Our mathematical analysis of $\WI\oplus\WII$ flows and their corresponding algorithms, relies on the theory of mean-field game and, more specifically, on the well-posedness of the MFG optimality conditions, consisting of a backwards Hamilton-Jacobi equation coupled with a forward continuity PDE. The  potential function for the optimal generative flow will solve  the HJ equation, while the variational derivative of the Wasserstein-1 proximal \eqref{eq:first:variation} is exactly the terminal condition of the HJ equation, which is well-defined for any distributions. We were able to express the optimality conditions of $\WI\oplus\WII$ generative flows in Theorem~\ref{thm:main}, meaning that the optimality conditions are at least well-defined. We then proved in Theorem~\ref{thm:uniqueness}, a uniqueness result for the MFG of our Wasserstein-1 and Wasserstein-2 proximal flow. This showed that the optimal generative flow is and unique, and therefore the optimization problem is well-posed, implying that it should be robust to different valid discretizations and approximations.

 We derived an adversarial training algorithm for learning  $\WI\oplus\WII$ proximal generative flows using the dual formulation of the Wasserstein-1 proximal. Due to this  adversarial structure, no forward-backward simulation of the flow is required, which is in contrast to likelihood-based training of CNFs.  The resulting generative flow is \emph{more robust} for learning datasets such as the  MNIST dataset without the use of any pre-trained autoencoders or tuning architectures as is typically required in CNFs \cite{chen2018neural,grathwohl2018ffjord,finlay2020train,onken2021ot,yang2019potential}.

 Finally, we also demonstrated the use  of the HJ optimality conditions as a new  \textit{real-time} criterion for optimality of solutions and termination of the algorithm. We found them to be more informative for determining when to stop training the flow rather than imposing them as additional regularizers in the learning objectives as in some related work \cite{onken2021ot,yang2019potential}.

An appealing attribute of $\WI\oplus \WII$ generative flows is the fact they can be trained with only forward simulations of the sample trajectories. For future work, we may explore new approaches for training \emph{stochastic} normalizing flows, which has been a challenge as backward SDEs are difficult to simulate efficiently \cite{tzen2019neural,hodgkinson2020stochastic}. Moreover, we may investigate $\WI\oplus\WII$ generative flows and its connections and applications to distributionally robust optimization (DRO) \cite{Rahimian2019DistributionallyRO}. As $\WI\oplus\WII$ generative flows are another example of an MFG-based generative model, we may rigorously study their robustness to implementation errors using PDE theory, following the style and approach of \cite{mimikos2024score}.

\section{Acknowledgement}
 The authors  are partially funded by AFOSR grant FA9550-21-1-0354. H. G., M.K. and L. R.-B. are partially funded by  NSF DMS-2307115. H. G. and M. K. are partially supported  by  NSF TRIPODS CISE-1934846.

\bibliography{bibliotheque.bib}
\bibliographystyle{unsrt}

\newpage
\appendix

\section{Experimental details}
\label{appendix:sec:experiment:detail}

\subsection{Compute resources}
Our experiment is computed using personal computer in the environment:  \texttt{Apple M2 8 cores} and \texttt{Apple M2 24 GB - Metal 3}.

\subsection{Dataset}
MNIST is a standard machine learning example where the dataset can be retrieved from some machine libraries. The original dataset consists of total 60000 samples of $28\times 28$ gray-scale images whose pixel values range from integers between 0 and 255. We normalize the pixel values and therefore locate samples in $[0,1]^{28 \times 28}$. Each sample represents one of handwritten digits from 0 to 9.  We use randomly chosen 6000 samples (10\% of the entire dataset) as training data.

\subsection{Neural network architecture}
\label{appendix:subsec:nn:architecture}

\begin{table}[ht]

\begin{minipage}{.49\linewidth}
\centering
\begin{tabular}{c}

\hline
CNN Discriminator \\
\hline
  $7 \times 7$ Conv, $1 \times 1$ stride  ($1 \rightarrow 8$) \\
  ReLU\\
  $2 \times 2$ max pool, $2 \times 2$ stride \\
  \hline
  $7 \times 7$ Conv, $1 \times 1$ stride ($8 \rightarrow 8$)  \\
  ReLU\\
  $2 \times 2$ max pool, $2 \times 2$ stride \\
  \hline
  Flatten with dimension $\ell_3$ \\
  \hline
  $W^4 \in \mathbb{R}^{\ell_3 \times 256}$, $b^4 \in \mathbb{R}^{256}$ \\
  ReLU\\
  \hline
  $W^5 \in \mathbb{R}^{256 \times 256}$, $b^5 \in \mathbb{R}^{256}$ \\
  ReLU\\
  \hline
  $W^6 \in \mathbb{R}^{256 \times 256}$, $b^6 \in \mathbb{R}^{256}$ \\
  ReLU\\
  \hline
  $W^7 \in \mathbb{R}^{256 \times 1}$, $b^7 \in \mathbb{R}$ \\
  Linear\\
\hline
\end{tabular}
\medskip
\end{minipage}
\begin{minipage}{.49\linewidth}
\centering
\begin{tabular}{c}

\hline
CNN Potential \\
\hline
  $7 \times 7$ Conv, $1 \times 1$ stride  ($1 \rightarrow 8$) \\
  ReLU\\
  $2 \times 2$ max pool, $2 \times 2$ stride \\
  \hline
  $7 \times 7$ Conv, $1 \times 1$ stride ($8 \rightarrow 8$)  \\
  ReLU\\
  $2 \times 2$ max pool, $2 \times 2$ stride \\
  \hline
  Spatial flatten with dimension $\ell_3$ \\
  \hline
  $W^4 \in \mathbb{R}^{(\ell_3 +1) \times 512}$, $b^4 \in \mathbb{R}^{512}$ \\
  ReLU\\
  \hline
  $W^5 \in \mathbb{R}^{512 \times 512}$, $b^5 \in \mathbb{R}^{512}$ \\
  ReLU\\
  \hline
  $W^6 \in \mathbb{R}^{512 \times 512}$, $b^6 \in \mathbb{R}^{512}$ \\
  ReLU\\
  \hline
  $W^7 \in \mathbb{R}^{512 \times 1}$, $b^7 \in \mathbb{R}$ \\
  Linear\\
\hline
\end{tabular}
\medskip
\end{minipage}

\caption{Neural network architectures of discriminator $\phi: \mathbb{R}^d \rightarrow \mathbb{R}$ and potential $\phi: \mathbb{R}^{d} \times \mathbb{R} \rightarrow \mathbb{R}$. We adapted the neural network architecture from the code for the potential flow generator GAN \cite{yang2019potential} to accommodate high-dimensional image examples. }
\label{table:nn architecture}
\end{table}


\end{document}